%% file: main.tex
\documentclass[iicol,pdflatex,sn-mathphys-num]{sn-jnl}
\usepackage{color}
\usepackage{makecell}
\usepackage{amsmath,amsfonts,amssymb}
\usepackage{mathtools}
\usepackage{pifont}
\usepackage{booktabs}       
\usepackage{graphicx}
\usepackage{multirow}
\usepackage{lineno} 
\linenumbers
\def\etal{\textit{et al. }}

\theoremstyle{thmstyleone}%
\newtheorem{theorem}{Theorem}
\newtheorem{corollary}[theorem]{Corollary}%

\theoremstyle{thmstyletwo}%
\newtheorem{remark}{Remark}%
\theoremstyle{thmstylethree}%
\newtheorem{definition}{Definition}%

\newcolumntype{A}{>{\centering\arraybackslash}p{0.2\columnwidth}}
\newcolumntype{B}{>{\centering\arraybackslash}p{0.15\columnwidth}}
\newcolumntype{C}{>{\centering\arraybackslash}p{0.12\columnwidth}}

\title[Article Title]{Learning Task-Aware Sampling with Shared Saliency through Density-Equalizing Mappings}
\author[1]{\fnm{Tsz Lok} \sur{Ip}}\email{enochitl@link.cuhk.edu.hk}
\author[2]{\fnm{Han} \sur{Zhang}}\email{hzhang863-c@my.cityu.edu.hk}
\author*[1]{\fnm{Lok Ming} \sur{Lui}}\email{lmlui@math.cuhk.edu.hk}

\affil*[1]{\orgdiv{Department of Mathematics}, \orgname{The Chinese University of Hong Kong}, \orgaddress{\city{Hong Kong}, \country{China}}}
\affil[2]{\orgdiv{Department of Mathematics}, \orgname{City University of Hong Kong}, \orgaddress{\city{Hong Kong}, \country{China}}}

\nolinenumbers

\begin{document}

\abstract{
In image and surface-based learning tasks, convolutional features are typically extracted using receptive fields that are sampled uniformly across the entire domain. However, informative structures are rarely distributed uniformly in practice and are often concentrated in localized regions. Such phenomena are particularly common in medical imaging, where pathological changes are spatially confined. Consequently, uniform convolution allocates equal computational effort to both informative and uninformative regions, resulting in inefficient feature extraction and suboptimal utilization of model capacity.
To address this issue, we propose a framework for task-adaptive sampling that dynamically redistributes computational attention according to the spatial importance of the data. Specifically, we introduce the Density-Equalizing Convolutional Neural Network (DECNN), which employs density-equalizing mappings to guide convolution through a learned density function. The density function encodes the relative importance of different regions and induces a transformation that enlarges informative areas while compressing less relevant ones. As a result, convolutional receptive fields are redistributed non-uniformly over the domain, enabling denser sampling in task-relevant regions.
By coupling this importance-driven transformation with convolution, DECNN performs adaptive feature extraction that focuses computational resources on informative structures. This leads to more efficient use of model capacity, yielding a lightweight yet expressive architecture while simultaneously producing an interpretable saliency map. Experiments on image classification and craniofacial surface analysis demonstrate that DECNN achieves competitive or superior performance with fewer parameters, accurately identifies task-relevant regions, and remains robust under complex geometric variations.
}


\maketitle


\section{Introduction}

In many image analysis tasks, data are processed at a predefined resolution, where a fixed number of feature samples can be extracted for downstream learning. A similar situation arises in surface-based learning, where geometric domains must be discretized into finite samples for computation. Under such settings, a fundamental yet often overlooked question is how to optimally allocate sampling resources for feature extraction. In most existing approaches, convolutional receptive fields are distributed uniformly across the domain to ensure unbiased spatial coverage. However, this strategy can be inefficient in practice, since informative features are rarely distributed evenly. Instead, discriminative patterns are often concentrated in localized regions, while large portions of the domain contribute relatively little to the learning task.

This issue is particularly evident in medical imaging, where pathological patterns are often localized within specific anatomical structures, while the surrounding regions remain largely consistent across subjects. For instance, neurodegenerative diseases such as Alzheimer’s disease primarily affect localized brain regions, and lung nodules in chest CT scans are confined to small areas. In such scenarios, uniformly distributing computational resources leads to inefficient feature extraction, as it fails to prioritize the most informative regions. Therefore, a key challenge is to design a task-aware sampling strategy that maximizes feature representation under a fixed computational budget.

To address this challenge, we aim to develop an optimal feature extraction mechanism that adapts to the underlying data distribution. The key idea is to guide feature extraction used for the convolution operation by using a learned measure of spatial importance, so that sampling is no longer uniform but instead aligned with task-relevant structures. In particular, we aim to construct an importance map that quantifies how much each region contains informative features across a collection of images. This learned measure then directly governs how sampling resources are allocated during feature extraction and learning, enabling the model to focus on informative regions while reducing redundancy in less relevant areas.

To this end, we propose a novel framework based on density-equalizing mappings, which provide a geometric way to realize task-aware sampling. Specifically, we introduce a density function defined on the domain to encode spatial importance. The corresponding density-equalizing mapping redistributes this density by diffusing it toward a uniform distribution, effectively transforming the domain such that regions of higher importance are expanded while less relevant areas are compressed. In practice, this transformation can be computed by solving a diffusion equation driven by the learned density function. Through this mechanism, the allocation of sampling points becomes adaptive and aligned with the task.

Building on this idea, we developed the Density-Equalizing Convolutional Neural Network (DECNN), which integrates this adaptive sampling mechanism into deep learning. At its core is the proposed Density-Equalizing Convolution (DEC), an adaptive convolution operator defined on simply connected open surfaces. The DEC operator is modulated by a learnable density map through the associated density-equalizing mapping, allowing the convolution to operate on a reparameterized domain that reflects feature importance. As a result, the network can utilize its receptive fields more efficiently, focusing on informative regions while reducing redundant computation. This leads to a lightweight yet expressive model that achieves improved performance without simply increasing model complexity.

We validate the proposed framework through a series of experiments. First, we apply DECNN to an image classification task in which only a small, unknown subregion contains task-relevant information. The results show that our method is able to learn an effective importance map that guides task-specific sampling, enabling the model to focus more on informative regions while maintaining classification performance with fewer trainable parameters. We further evaluate the method on a craniofacial analysis task involving 3D facial surfaces. In this setting, the learned importance map consistently highlights anatomically meaningful regions that are critical for distinguishing different classes. By optimizing sampling and computational resources according to this learned measure, the model achieves superior classification accuracy compared to several existing approaches, even when operating with a limited number of feature points.

To summarize, the main contributions of this work are as follows:
\begin{itemize}
\item We introduce Density-Equalizing Convolution, a novel family of geometrically interpretable convolution operators induced by a density function defined on simple-connected open surfaces.

\item We developed a theoretical analysis to show that the associated density-equalizing mapping redistributes sampling points according to a learned density function, such that regions with higher density values receive more samples. This establishes a principled connection between spatial importance and sampling distribution.

\item We propose the Density-Equalizing Convolutional Neural Network, a general framework that integrates DEC into standard 2D CNN architectures, enabling efficient learning on simply connected Riemann surfaces.

\item We demonstrate that DECNN can automatically identify shared task-relevant regions by learning an optimal density function, which serves as a geometrically meaningful saliency map and guides task-aware feature extraction.
\end{itemize}

\section{Related Works}

\subsection{Deformable Convolution and Geometric Learning}
Spatially deformable convolution has been shown to enhance learning capability and has received increasing attention. Jeon \etal~\cite{jeon2017active} proposed Active Convolution, incorporating a trainable attention mechanism into the convolution process. Related methods, such as Spatial Transformer Networks and their variants~\cite{jaderberg2015spatial,zhang2024learning,zhang2025deformation}, introduce learnable transformation modules that warp input feature maps according to trainable parameters. Subsequently, Dai \etal~\cite{dai2017deformable} proposed Deformable Convolution (DCN), which adds learnable offsets to convolutional kernels, enabling spatial adaptation. While DCN attracted much attention, it faces challenges in handling large deformations and achieving occlusion invariance. Variants such as Deformable Convolution v2~\cite{zhu2019deformable} and Deformable RoI Pooling~\cite{dai2017deformable} were introduced to mitigate these limitations, though Luo \etal~\cite{luo2016understanding} demonstrated that not all pixels contribute equally to the final DCN output. These methods, however, are primarily designed for Euclidean domains.

Geometric deep learning generalizes neural networks to non-Euclidean domains such as manifolds and meshes, where data lie on curved spaces rather than regular grids. Unlike Euclidean convolution~\cite{lecun1995convolutional}, convolution on manifolds must account for curvature and the lack of global translation invariance. Early works defined local convolution operators using intrinsic geometric structures, establishing the foundation for learning on curved surfaces~\cite{masci2015geodesic, boscaini2016learning}. Subsequent approaches explored intrinsic and mesh-based convolutions for shape analysis and registration~\cite{bouritsas2019neural, gong2019spiralnet++, hanocka2019meshcnn}, while Parallel Transport Convolution~\cite{schonsheck2022parallel} improved geometric consistency and filter localization. More recently, Zhang \etal~\cite{zhang2025quasi} introduced quasi-conformal convolution, which leverages quasi-conformal maps to adapt convolutional kernels, enabling geometry-aware, locally adaptive filtering on Riemann surfaces.

\subsection{Computational Quasi-Conformal Mapping and Density-Equalizing Mapping}
Computational quasi-conformal mapping provides a powerful framework to control geometric deformation, with successful applications in image processing~\cite{lam2014landmark} and surface analysis~\cite{levy2002least,gu2004genus}. Leveraging the Beltrami representation, mappings between domains can preserve desirable geometric properties such as bijectivity and smoothness by controlling the associated Beltrami coefficients. Motivated by the need to preserve different geometric features, a variety of parameterization methods have been developed~\cite{gu2003global}, and such representations have also found success in computational fabrication~\cite{Soliman:2018:OCS,Crane:2013:RFC,panetta2019x}. With their ability to handle large deformations, quasi-conformal methods have been applied to registration~\cite{lam2014landmark,choi2015fast}, restoration~\cite{zhang2025deformation}, and segmentation under topological and convexity constraints~\cite{zhang2025qis,zhang2024learning}. Furthermore, quasi-conformal analysis has been employed for deformation studies with uncertainties, enabling robust medical image analysis~\cite{zhang2022nondeterministic,zhang2022new}.

An important component in our method is the use of density-equalizing mappings, which rescale spatial domains according to prescribed density functions~\cite{gastner2004diffusion}. Choi and Rycroft~\cite{choi2018density, choi2020area} extended DE mappings to general simply-connected open surfaces and volumetric images~\cite{choi2021volumetric}, enabling applications on complex 3D geometries. To further improve geometric fidelity, Lyu \etal~\cite{lyu2024bijective} combined quasi-conformal mapping with DE mappings, producing bijective density-equalizing quasi-conformal mappings for simply- and multiply-connected surfaces, as well as spherical genus-0 shapes~\cite{lyu2024spherical}. More recently, Yao \etal~\cite{yao2026toroidal} generalized these techniques to toroidal genus-1 surfaces, extending the applicability of DEQ mappings to increasingly complex topologies.

\section{Mathematical Background}

In this section, we introduce the mathematical foundations underlying our framework. In particular, quasi-conformal mappings provide a flexible yet well-controlled way to describe local geometric distortions and serve as the theoretical basis of this work. Density-equalizing mappings are the primary tool we employ, enabling the redistribution of spatial importance across the domain. We also present the definition of manifold convolution, which provides background for categorizing our proposed Density-Equalizing Convolution (DEC).

\subsection{Quasi-Conformal Geometry}

\begin{definition}[Quasi-conformal map]
A quasi-conformal map is a map $f: \mathbb{C} \rightarrow \mathbb{C}$ that satisfies the Beltrami equation
\begin{equation}
\frac{\partial f}{\partial \bar{z}}=\mu(z) \frac{\partial f}{\partial z}
\label{eq:beleq}
\end{equation}
for some complex-valued function $\mu$ satisfying $\|\mu\|_{\infty}<1$ and $\frac{\partial f}{\partial z}$ is non-vanishing almost everywhere. The complex partial derivatives are given by
\begin{equation}
\frac{\partial f}{\partial z}:=\frac{1}{2}\left(\frac{\partial f}{\partial x}-i \frac{\partial f}{\partial y}\right) 
\quad \text{ and } \quad 
\frac{\partial f}{\partial \bar{z}}:=\frac{1}{2}\left(\frac{\partial f}{\partial x}+i \frac{\partial f}{\partial y}\right).
\end{equation}

\end{definition}

$\mu$ is called the \emph{Beltrami representation}, or the \emph{Beltrami coefficient}, of the quasi-conformal map $f$. Note that if $\mu \equiv 0$, Eq~\eqref{eq:beleq} becomes the Cauchy-Riemann equation, and so $f$ becomes conformal. Infinitesimally, we can express $f$ in a neighbourhood around a point $p$ as
\begin{equation}
\label{eq:qc_local}
\begin{aligned}
f(z) &=f(p)+f_{z}(p)(z-p)+f_{\bar{z}}(p) \overline{(z-p)} \\
&=f(p)+f_{z}(p)((z - p)+\mu(p) \overline{(z-p)}),
\end{aligned}
\end{equation}
hence $f$ is conformal around the neighbourhood of $p$ if and only if $\mu(p) = 0$.

Furthermore, $\mu$ can be seen as a measure of non-conformality. While Eq~\eqref{eq:qc_local} shows that the quasi-conformal mapping $f$ maps a small circle centered at $p$ to a small eclipse centered at $f(p)$, the angle of maximal magnification is $\arg (\mu(p)) / 2$ with magnifying factor $|f_{z}(p)|(1 + |\mu(p)|)$; the maximal shrinking is along the orthogonal angle $(\arg (\mu(p))-\pi) / 2$ with shrinking factor $|f_{z}(p)|(1 - |\mu(p)|)$, which is illustrated in Fig~\ref{fig:qcmap}. Then, the maximal quasi-conformal dilation of $f$ is given by
\begin{equation}
K=\frac{1+\|\mu\|_{\infty}}{1-\|\mu\|_{\infty}}.
\end{equation}

\begin{figure}[t]
    \centering
    \includegraphics[width=0.4\textwidth]{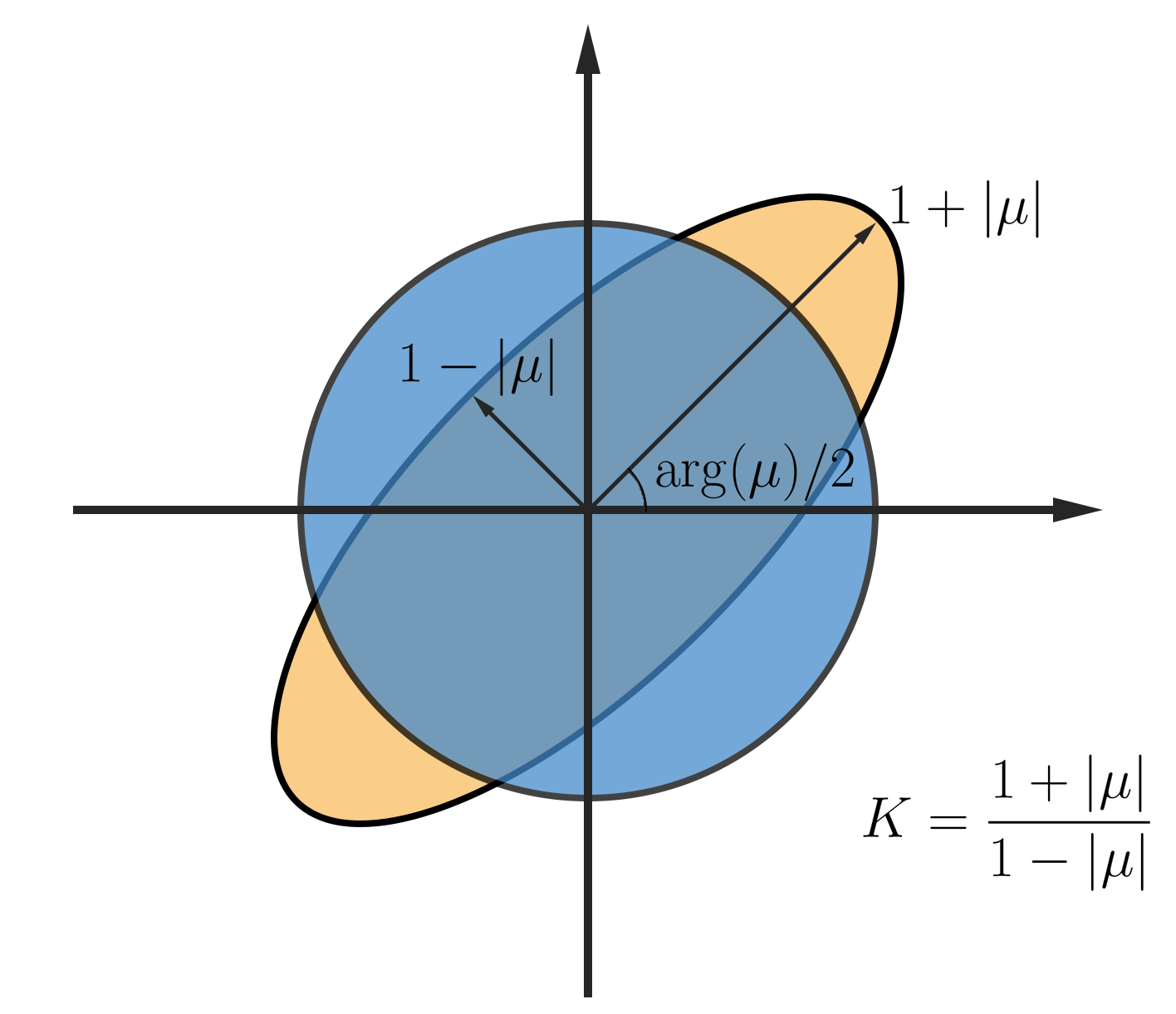}
    \caption{Illustration of Beltrami coefficient being the measure of conformality distortion of a quasi-conformal map}
    \label{fig:qcmap}
\end{figure}

Quasi-conformal geometry is useful for studying deformations, not only because it provides an effective way to describe angle distortion, but also because it supports smooth one-to-one mappings by its diffeomorphism property. By introducing a norm constraint on $\mu$, the bijectivity of $f$ can be ensured, as illustrated in the following theorem:
\begin{theorem}
If $f: \mathbb{C} \rightarrow \mathbb{C}$ is a $C^{1}$ map. Define 
\begin{equation}
\mu=\frac{\partial f}{\partial \bar{z}} / \frac{\partial f}{\partial z}.
\end{equation}
If $\mu$ satisfies $\left\|\mu_{f}\right\|_{\infty}<1$, then $f$ is bijective.
\end{theorem}

\subsection{Density-Equalizing Mappings}
\label{sec:de_map}
Given a population density function $\rho$ defined on a 2D planar domain $\Omega$, Gastner and Newman~\cite{gastner2004diffusion} proposed a method to compute density-equalizing mappings based on diffusion. Through the diffusion process, the domain is deformed so that the density is eventually equalized. The key feature of such mapping is that regions with high densities are enlarged and regions with low densities are shrunk. 

Since diffusion follows the gradient of $\rho$, let $\boldsymbol{j} = - \nabla \rho$ be the flux by Fick's law of diffusion. By conservation of local population, 
\begin{equation}
    \frac{\partial \rho}{\partial t} = - \nabla \cdot \boldsymbol{j} =\Delta \rho.
\end{equation}
Note that $\boldsymbol{j} = \boldsymbol{v} \cdot \rho$, where $\boldsymbol{v}$ is the velocity, we have
\begin{equation}
    \boldsymbol{v} = -\frac{\nabla \rho}{\rho}.
\end{equation}
Therefore, the displacement of any point can be calculated by
\begin{equation}\label{displacement}
    \boldsymbol{x}(t) = \boldsymbol{x}(0) + \int^{t}_{0} \boldsymbol{v}(\boldsymbol{x}(t'),t') \, dt'.
\end{equation}
Note that the points flow from high-density regions to low-density regions. The density $\rho$ will be equalized to $\bar{\rho}$ when $t \to \infty$, where
\begin{equation}
    \bar{\rho} = \frac{1}{|\Omega|} \int_\Omega \rho(\boldsymbol{x}, 0) \, d\boldsymbol{x}.
    \end{equation}

\subsection{Manifold Convolution}

In the Euclidean case, convolution is defined by translating the kernel $k$ using the displacement vector $x-y$. However, this notion does not directly extend to manifolds, where a suitable definition of displacement must first be established. Thus, we first present the formal definition for the displacement function~\cite{zhang2025quasi}:

\begin{definition}[Displacement function and displacement vector]
    Let $\mathcal{M}$ be a Riemannian $n$-manifold, and $U \subseteq \mathcal{M}$ be a subset of $\mathcal{M}$. A function $d: U\times U \to \mathbb{R}^n$ is a \textit{displacement function} on $U$ if it satisfies the following properties:
    \begin{enumerate}
        \item For all $p, q \in U$, $d(p,q) = 0$ if and only if $p=q$.
        \item For all $p, q, r \in U$, $d(p,r) = d(p,q) + d(q,r)$.
    \end{enumerate}
    Then, the vector $d(p, q)$ is referred to as the \textit{displacement vector} from $p$ to $q$.
    
    Moreover, if the functions $d(\cdot, q_0)$ and $d(p_0, \cdot)$ are orientation-preserving homeomorphisms from $U$ to subsets of $\mathbb{R}^n$ depending on $p_0$ and $q_0$ for all fixed $p_0, q_0 \in U$, then the displacement function $d$ is said to be \textit{regular}.
\end{definition}

The above displacement function is introduced to replace the standard expression $ x-y $ in $ \mathbb{R}^n $ by mimicking translation symmetry \cite{schonsheck2022parallel}. It ensures consistent and additive displacement between points, enabling the kernel to be effectively translated across the manifold.

With the displacement function $d$ on $\mathcal{M}$, the definition of manifold convolution is given as:
\begin{definition}[Manifold convolution]\label{def:mani_conv}
Let $\mathcal{M}$ be a Riemannian $n$-manifold with a metric $g$, and let $h: \mathcal{M} \to \mathbb{R}$ be a manifold function with a kernel function $k: \mathbb{R}^{n}\to \mathbb{R}$. The convolution of $h$ and $k$ on $\mathcal{M}$ is defined as:
\begin{equation}\label{eq:mani_conv}
(h \ast_{\mathcal{M},d,g} k)(p) = \int_{\mathcal{M}} h(q) k(d(p, q)) \, dq,
\end{equation}
where:
\begin{itemize}
    \item $p, q \in \mathcal{M}$,
    \item $d: \mathcal{M}\times\mathcal{M} \to \mathbb{R}^n $ is a global displacement function on $\mathcal{M}$.
\end{itemize}
For simplicity, we denote $*_{\mathcal{M}, d,g}$ as $*_{d,g}$.
Moreover, $*_{d,g}$ is said to be regular if the displacement function $d$ is regular.
\end{definition}

Convolution on manifolds is challenging due to curvature, which prevents a straightforward notion of kernel shifting as in Euclidean space. To address this, we define convolution on a Riemann surface via its 2D parametric domain:
\begin{definition}[Parametrized manifold convolution]
\label{def:parameterizedconv}
    Let $\mathcal{M}$ be a Riemannian $n$-manifold, and let $h: \mathcal{M} \to \mathbb{R}$ be a manifold function with a kernel function $k: \mathbb{R}^{n}\to \mathbb{R}$. Suppose there exists a bijective parametrization $\phi:\Omega \to \mathcal{M}$, where $\Omega \subset \mathbb{R}^n$. The parametrized manifold convolution of $h$ and $k$ with respect to $\phi$ is defined as:
    \begin{equation}
        (h *_\phi k)(p) = \int_\Omega h(\phi(y))k(\phi^{-1}(p)-y)dy.
    \end{equation}    
\end{definition}

This definition simplifies manifold convolution by expressing it on a Euclidean domain in a computational sense. The following lemma establishes its equivalence to the displacement-based formulation in manifold convolution by showing that a displacement function induces a corresponding parameterization:
\begin{theorem}\label{thm:equiv_conv}
    Let $\mathcal{M}$ be a Riemannian $n$-manifold and $d$ be a displacement function on $\mathcal{M}$. Then there exists a bijective parametrization $\phi: \Omega \to \mathcal{M}$, where $\Omega \subset \mathbb{R}^n$, along with a metric $g$ of $\mathcal{M}$, such that $*_{d, g} = *_\phi$. Conversely, for any bijective parametrization $\phi: \Omega \to \mathcal{M}$, there exists a displacement function $d$ on $\mathcal{M}$ and a metric $g$ of $\mathcal{M}$ such that $*_{d, g} = *_\phi$.
\end{theorem}

\section{Proposed method}
\label{sec:dec_all}
In this section, we present the proposed model. We first introduce the Density-Equalizing Convolution (DEC), a novel convolution operator that enables task-aware sampling of feature points through an explicitly learned importance map induced by a density-equalizing function. Building on this operator, we then develop the Density-Equalizing Convolutional Neural Network (DECNN), which is constructed upon a learnable density map.

\subsection{Density-Equalizing Convolution}
\label{sec:dec}
The proposed DEC is to introduce a mechanism to define convolution operations that are explicitly guided by a population density function encoding spatial importance. This population density function is a positive scalar field that reflects the relative significance of different regions across the domain, and directly governs how the convolution adapts to the underlying data.

\begin{definition}[Density-Equalizing Convolution]
\label{def:de_conv}
    Let $\mathcal{M} \subset \mathbb{R}^3$ be a 2-manifold, and let $h: \mathcal{M} \to \mathbb{R}$ be a manifold function with a kernel function $k: \mathbb{R}^{n}\to \mathbb{R}$. Suppose there exists a parametrization $\phi: \Omega \to \mathcal{M}$ and a density-equalizing mapping $f_\rho: \Omega \to \Omega$ induced by density function $\rho$, where $\Omega \subset \mathbb{R}^2$. The density-equalizing convolution of $h$ and $k$ with respect to $\phi$ and $f_\rho$ is defined as:
    \begin{equation}
        (h \ast_{\phi, f_\rho} k)(p)
        = \int_{\Omega}h(\phi(f_\rho^{-1}(y')))k(f_\rho \circ \phi^{-1}(p) - y') dy'.
        \label{eq:de_conv}
    \end{equation}
\end{definition}

Density-equalizing convolution belongs to the set of parametrized convolution defined in Definition~\ref{def:parameterizedconv} by having the parameterization as $\phi \circ f^{-1}$. Thus, according to Theorem~\ref{thm:equiv_conv}, we could have the following remark:
\begin{remark}
\label{rmk:dec=2d}
Under the condition of Definition~\ref{def:de_conv}, the DEC of $h$ and $k$ can be written in the manifold convolution form given in Definition~\ref{def:mani_conv} as:
\begin{equation}   
\begin{aligned}
    (h \ast_{\phi, f_\rho} k)(p)
    &= \int_{\Omega}h(\phi(f_\rho^{-1}(y')))k(f_\rho\circ\phi^{-1}(p) - y') dy'\\
    &= \int_{\mathcal{M}}h(q)k(f_\rho\circ\phi^{-1}(p) - f_\rho\circ\phi^{-1}(q)) dq \\
    &= (h \ast_{d,g }k)(p).
    \label{eq:qcc_rewrite}  
\end{aligned}
\end{equation}
where:
\begin{itemize}
    \item $p,q \in \mathcal{M}$,
    \item $x=\phi^{-1}(p), y=\phi^{-1}(q) \in \Omega$,
    \item $x'=f_\rho(x), y'=f_\rho(y) \in \Omega$,    
    \item $d(p,q) = f_\rho \circ \phi^{-1}(p) - f_\rho \circ \phi^{-1}(q)$ is the displacement function,
    \item $g = (f_\rho \circ \phi^{-1})^* g_{\mathbb{R}^2}$ is the Riemannian metric of $\mathcal{M}$.
\end{itemize}    
Besides, it can also be rewritten into a general 2D convolution by having the pullback function $\tilde{h} = h \circ \phi$, and the transformed pullback function $\hat{h}=\tilde{h}\circ f_\rho ^{-1}$, which leads to
\begin{equation}
\begin{aligned}
    (h \ast_{\phi, f_\rho} k)(p)
    &= \int_{\mathcal{M}}h(q)k(f_\rho\circ\phi^{-1}(p) - f_\rho\circ\phi^{-1}(q)) dq\\
    &= \int_{\Omega} h\circ\phi(y) k(f_\rho(x) - f_\rho(y)) \, df_\rho(y) \\
    &= \int_{\Omega} \tilde{h}(y)k(f_\rho(x) - f_\rho(y)) \, df_\rho(y)\\
    &= \int_{\Omega} \tilde{h}\circ f_\rho^{-1}(y')k(x' - y') \, dy' \\
    &= \int_{\Omega} \hat{h}(y')k(x' - y') \, dy'\\
    &= (\hat{h} \ast k)(x').
    \label{eq:qcc2plain}
\end{aligned}
\end{equation}
\end{remark}

\subsection{Optimal Feature Extraction}
\label{sec:opt_disc}

In our framework, convolution is performed on a reparameterized domain induced by the density-equalizing mapping together with the parameterization. In this section, we revisit this process in a discrete setting, which provides a more convenient way to analyze the distribution of sampled features. In particular, this distribution can be interpreted through the notion of population in density-equalizing mapping theory, defined as the integral of the density function over a given region.

Let the image grid be denoted by $\{x_{i,j} : i = 1,\cdots,I,\; j = 1,\cdots,J\}$, representing pixel locations in the parameter domain. As previously presented, $\phi$ denotes the parameterization, $f_\rho$ denotes the density-equalizing mapping. Correspondingly, $\tilde{h} = h \circ \phi$, $\hat{h} = \tilde{h}\circ f_\rho^{-1}$. The discretized image is then given by $H = (H_{i,j})$, where
\begin{equation}
    H_{i,j} = \hat{h}(x_{i,j}) = \tilde{h}(f_\rho^{-1}(x_{i,j})).
\end{equation}
Let the convolution kernel be denoted by $K = (K_{m,n})$. The discrete convolution can be written as
\begin{equation}
\begin{aligned}
    (h \ast_{\phi,f_\rho} k)(q_{i,j})
    & = (\hat{h} \ast k)(x_{i,\,j})\\
    &\approx \sum_m \sum_n H_{i-m,\,j-n}\, K_{m,n} \\
    &= \sum_m \sum_n \hat{h}(x_{i-m,\,j-n})\, K_{m,n} \\
    &= \sum_m \sum_n h(q_{i-m,\,j-n})\, K_{m,n},
\end{aligned}
\end{equation}
where $q_{i,j} = f_\rho^{-1}(x_{i,j})$ are sampling points on the manifold domain.

This formulation shows that determining the density-equalizing mapping $f_\rho$ is equivalent to selecting the sampling locations $\{q_{i,j}\}$ in the original domain. Consequently, the learned mapping directly controls the spatial distribution of samples, allowing the convolution to adapt to the underlying importance structure of the data.

\begin{theorem}[Density-Induced Sampling Distribution (Probabilistic Formulation)]
\label{thm:DE_sampling}
Let $\Omega \subset \mathbb{R}^2$ be a rectangular domain, and let $X$ be a random variable uniformly distributed over $\Omega$, i.e., $x \sim \mathcal{U}(\Omega)$. Let $\rho : \Omega \to \mathbb{R}^+$ be a population density function, and $f_\rho : \Omega \to \Omega$ the density-equalizing mapping induced by $\rho$. Define the transformed random variable
\begin{equation}
Y = f_\rho^{-1}(X).
\end{equation}
Then $Y$ admits a probability density function $p_\rho$ on $\Omega$ defined by
\begin{equation}
p_\rho(y) = \frac{1}{P}\, \rho(y),
\end{equation}
where $P=\int_\Omega \rho(x)dx$ is the total population in $\Omega$.

\noindent In particular, for all $y_1, y_2 \in \Omega$,
\begin{equation}\label{eq:thm_rela1}
\rho(y_1) > \rho(y_2)
\;\;\Longrightarrow\;\;
p_\rho(y_1) > p_\rho(y_2).
\end{equation}
\end{theorem}

\begin{proof}
\noindent Let $x = f(y)$, and $\bar{\rho} = \frac{P}{|\Omega|}$ be the final equalized density value. By conservation of total population, for any measurable subdomain $\Omega' \subset \Omega$,
\begin{equation}
\begin{aligned}
    \int_{\Omega'} \rho(y) \, dy &= \int_{f(\Omega')} \bar{\rho} \, dx \\
    &= \int_{\Omega'} \bar{\rho} \left|\det Df_\rho(y)\right| \, dy
\end{aligned}
\end{equation}
Since $\Omega'$ is arbitrary, we have
\begin{equation}
     \left|\det Df_\rho(y)\right| = \frac{\rho(y)}{\bar{\rho}}
\end{equation}
for almost every $y \in \Omega$.

Denote $p_0$ as the probability density function of $X$, then $Y$ admits a probability density function $p_\rho$ on $\Omega$ given by
\begin{equation}
\begin{aligned}
p_\rho(y) 
&\stackrel{\text{a.e.}}{=} p_0(x) \left|\det Df_\rho(y)\right| \\
&\stackrel{\text{a.e.}}{=} \frac{1}{|\Omega|}\, \frac{\rho(y)}{\bar{\rho}}\\
&\stackrel{\text{a.e.}}{=} \frac{1}{P}\, \rho(y).
\end{aligned}
\end{equation}
Hence $p_\rho(y)=\frac{1}{P}\, \rho(y)$ is a probability density function of $Y$.
Since $P$ is positive, it is straightforward to see that the relation Eq~\eqref{eq:thm_rela1} holds.
\end{proof}

\begin{corollary}
\label{cor:DE_sampling1}
    \noindent For any measurable subdomains $\Omega_1, \Omega_2 \subset \Omega$, if
    \begin{equation}
    \frac{1}{|\Omega_1|}\int_{\Omega_1} \rho(y)\, dy >
    \frac{1}{|\Omega_2|}\int_{\Omega_2} \rho(y)\, dy,
    \end{equation}
    then
    \begin{equation}\label{eq:P_ineq}
    \frac{1}{|\Omega_1|}\mathbb{P}(y \in \Omega_1)
    >
    \frac{1}{|\Omega_2|}\mathbb{P}(y \in \Omega_2).
    \end{equation}
\end{corollary}

\begin{proof}
Note that $p_\rho(y) = \frac{1}{P}\, \rho(y)$, then
\begin{equation}
\begin{aligned}
\frac{1}{|\Omega_1|}\int_{\Omega_1} \rho(y)\, dy &>
\frac{1}{|\Omega_2|}\int_{\Omega_2} \rho(y)\, dy \\
\frac{P}{|\Omega_1|}\int_{\Omega_1} p_\rho(y)\, dy &>
\frac{P}{|\Omega_2|}\int_{\Omega_2} p_\rho(y)\, dy \\
\frac{1}{|\Omega_1|}\mathbb{P}(y \in \Omega_1) &>
\frac{1}{|\Omega_2|}\mathbb{P}(y \in \Omega_2).
\end{aligned}
\end{equation}
Therefore, \eqref{eq:P_ineq} holds.
\end{proof}

\begin{corollary}
\label{cor:DE_sampling2}
    Let $\rho_i : \Omega \to \mathbb{R}^+$ be population density functions defined on $\Omega$ and $y_i = f_{\rho_i}^{-1} (x)$ for $i = 1, 2$. Suppose there exists $\Omega' \subset \Omega$ such that $\rho_1 > \rho_2$ on $\Omega'$, $\rho_1 = \rho_2$ on $\Omega \backslash \Omega'$, and both $\Omega$ and $\Omega \backslash \Omega'$ are of positive measures, then
    \begin{equation}
        \mathbb{P}(y_1 \in \Omega') >
        \mathbb{P}(y_2 \in \Omega').
    \end{equation}
\end{corollary}

\begin{proof}
Let $P_{\Omega'}^i = \int_{\Omega'} \rho_i(y) \, dy$ and $P_{\Omega \backslash \Omega'}^i = \int_{\Omega \backslash \Omega'} \rho_i(y) \, dy$ for $i=1,2$. Note that 
\begin{gather}
    \label{eq:P_i_rel1}
    P_{\Omega'}^1 > P_{\Omega'}^2 >0, \\
    \label{eq:P_i_rel2}
    P_{\Omega \backslash \Omega'}^1 = P_{\Omega \backslash \Omega'}^2 >0. 
\end{gather}
For simiplicity, write $P_{\Omega \backslash \Omega'} = P_{\Omega \backslash \Omega'}^i$. Suppose $Y_i=f^{-1}(X)$ are the transposed random variables with probability density functions $p_{\rho_i}$ for $i=1,2$, then
\begin{equation}
\begin{aligned}
&\mathbb{P}(y_i \in \Omega')\\
&=\int_{\Omega'}p_{\rho_i}(y) \, dy \\
&= \frac{1}{P_{\Omega \backslash \Omega'} + P_{\Omega'}^i} \int_{\Omega'}\rho_i(y) \, dy \\
&= \frac{P_{\Omega'}^i}{P_{\Omega \backslash \Omega'} + P_{\Omega'}^i}.
\end{aligned}
\end{equation}
Hence, we have
\begin{equation}
\begin{aligned}
&\mathbb{P}(y_1 \in \Omega') - \mathbb{P}(y_2 \in \Omega') \\
&=
\frac{P_{\Omega'}^1}{P_{\Omega \backslash \Omega'} + P_{\Omega'}^1} - \frac{P_{\Omega'}^2}{P_{\Omega \backslash \Omega'} + P_{\Omega'}^2} \\
&= \left(1 -\frac{P_{\Omega \backslash \Omega'}}{P_{\Omega \backslash \Omega'} + P_{\Omega'}^1} \right) - \left(1 -\frac{P_{\Omega \backslash \Omega'}}{P_{\Omega \backslash \Omega'} + P_{\Omega'}^2} \right) \\
&= P_{\Omega \backslash \Omega'} \left( \frac{1}{P_{\Omega \backslash \Omega'} + P_{\Omega'}^2} - \frac{1}{P_{\Omega \backslash \Omega'} + P_{\Omega'}^1} \right) \\
&>0 \qquad\text{by \eqref{eq:P_i_rel1}, \eqref{eq:P_i_rel2}}.
\end{aligned}
\end{equation}
Therefore, we have $\mathbb{P}(y_1 \in \Omega') > \mathbb{P}(y_2 \in \Omega')$.
\end{proof}

Theorem \ref{thm:DE_sampling} provides the building block for achieving task-aware sampling through the density-equalizing mapping. Starting from a uniform grid defining the image domain in a discrete sense, the mapping $f_\rho$ redistributes sampling locations according to the learned density $\rho$, which encodes spatial importance. As a result, salient features are sampled with higher probability, whereas non-salient feature points are less likely to be captured. Corollary \ref{cor:DE_sampling1} generalizes this idea from pointwise probabilities to regional behavior, illustrating that regions with higher importance are assigned more sampling points, while less informative regions are sparsely sampled. To further analyze how sampling probabilities vary within a single region under different density functions, Corollary \ref{cor:DE_sampling2} states that an increase in density values of a predefined region will lead to more sampling points in that region, if the density of other regions is kept constant. Most importantly, the total number of sampling points remains unchanged; only their spatial distribution is adapted. 

These properties are particularly beneficial for feature extraction in neural networks. Since convolution operates over a fixed number of sampling locations (i.e., a fixed receptive field size and computational budget), the ability to reallocate these sampling points allows the model to focus its capacity on more informative regions without increasing complexity. In effect, the receptive field becomes spatially adaptive: it is denser in important regions and coarser elsewhere.

From this perspective, learning the density function $\rho$ is equivalent to learning how to optimally distribute sampling resources for a given task. Rather than uniformly covering the domain, the model concentrates its feature extraction on regions that contribute most to the task objective. This leads to more meaningful feature representations, as important structures are captured with higher resolution, while redundant areas are compressed. Based on this principle, we step into the details about the architecture of the overall architecture.

\section{Density-Equalizing Convolutional Neural Network}

\begin{figure*}[t]
    \centering
    \includegraphics[width=0.9\textwidth]{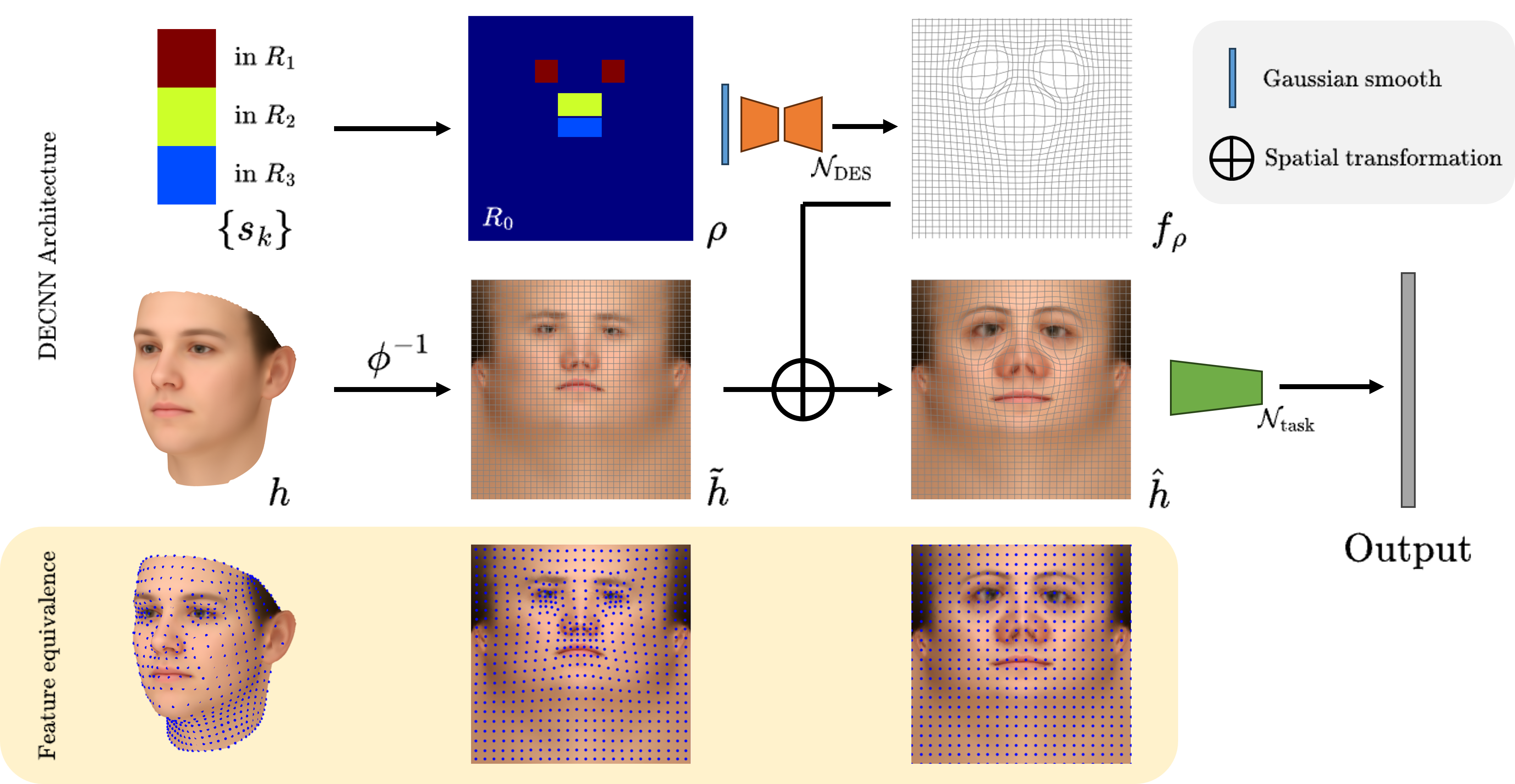}
    \caption{DECNN is constructed by 3 main steps: (1) constructing a DE mapping $f_\rho$ based on estimated density parameters $\{s_k'\}$ using a pretrained solver $\mathcal{N}_\text{DES}$, (2) transforming the parameterized 2D input into a domain focusing on salient regions by $f_\rho$, and (3) applying a 2D CNN $\mathcal{N}_\text{task}$. The last row illustrates that these operations are mathematically equivalent to sampling more feature points in those regions on the manifold.}
    \label{fig:decnn}
\end{figure*}

Starting from the defined Density-Equalizing Convolution (DEC), in this section, we introduce the Density-Equalizing Convolutional Neural Network (DECNN), which is capable of performing adaptive sampling during learning tasks on simply-connected open surfaces, as summarized in Fig~\ref{fig:decnn}. We will begin by discussing the generation of the density-equalizing mapping, then how to conveniently apply the corresponding convolution operator to form a DEC Layer with the help of widely developed 2D convolution, and finally, integrate these layers into a complete network.

\subsection{Density-Equalizing Mapping Solution}
\label{sec:de_map_sol}

In this work, the central component is the transformation induced by a density-equalizing mapping. For both images and surfaces, this mapping is constructed in the parameterized domain from a learned density function $\rho : \Omega \rightarrow \mathbb{R}^+$ that explicitly encodes the relative importance of different spatial regions. Intuitively, regions with higher density are treated as more informative and are therefore expanded under the transformation, while less relevant regions are compressed. The resulting mapping redistributes the domain such that the transformed density becomes approximately equal. After such a transformation, applying a general 2D convolution on the deformed image is equivalent to applying a deformable convolution induced by density-equalizing convolution on the original parameterized domain indicated by Remark~\ref{rmk:dec=2d}.

To make this mechanism task-adaptive, we introduce a single shared density function that governs the transformation across the entire dataset. This shared structure enables the model to consistently emphasize task-relevant regions while preserving a coherent global deformation. Specifically, as illustrated in Fig~\ref{fig:decnn}, we first parameterize each manifold image $h$ by a mapping $\phi$ yielding a parameterized image $\tilde{h} = h \circ \phi$ defined on $\Omega$, then construct a density-equalizing mapping $f_\rho$ from the density function $\rho$, which is used to deform $\tilde{h}$. By learning $\rho$ from data and sharing it across the dataset, the model can adaptively modulate the deformation to emphasize informative structures in a task-dependent manner. 

Furthermore, supported by Remark~\ref{rmk:dec=2d}, performing standard 2D convolution on a regular grid over the transformed image $\hat{h}$ is equivalent to applying convolution on $\tilde{h}$ with a task-adaptive sampling pattern induced by $f_\rho$. In this view, together with Theorem~\ref{thm:DE_sampling}, the receptive field is effectively deformed according to the learned density in a probability sense, allowing the convolution to focus more on informative regions while maintaining global consistency.

While the proposed framework is designed to handle smooth density functions and deformation mappings, many practical datasets are naturally composed of semantically distinct regions. For instance, in medical imaging, anatomical structures partition the image into regions with different functional roles, while in facial analysis, a face can be decomposed into components such as the eyes, nose, mouth, and surrounding background.

Incorporating such partitions derived from prior domain knowledge, particularly in medical applications, not only leads to more interpretable saliency maps that better reflect the underlying task or diagnosis, but also reduces the degrees of freedom in the learning process. This structured constraint, in turn, facilitates more efficient optimization when estimating the density function.

To incorporate such prior knowledge, we introduce a region prior $\mathcal{R}$ forming a partition of the domain $\Omega$,
\begin{equation}
\Omega = \cup_{k=0}^K R_k.
\label{eq:partition}
\end{equation}
We model the density function $\hat{\rho}$ as a piecewise constant function with respect to this partition, such that
\begin{equation}
\hat{\rho} := \sum\limits_{k=0}^K s_k \boldsymbol{1}_{{R_k}},
\end{equation}
where each $s_k \in \mathbb{R}^+$ is a optimizable constant associated with the fixed region $R_k$ according to prior knowledge. $\boldsymbol{1}_{{R_k}}$ is the indicator function for region $R_k$.

To build up the density function, which is a measure of importance, we expect the informative regions $R_k$ to correspond to a large density value $s_k \gg s_0$, where $R_0$ is set to be the background region. To exaggerate the difference between the informative and the remaining that are relatively less important, we apply a softmax-like transformation:
\begin{equation}
\begin{aligned}
    s'_k &:= \frac{M e^{s_k}}{\sum_{j=1}^K e^{s_j}}+1, \\
    \hat{\rho}'&:= \sum\limits_{k=0}^K s'_k \boldsymbol{1}_{{R_k}},
\end{aligned}
\end{equation}
where $M$ is a positive learnable sizing coefficient. 

While the piecewise constant density function provides an efficient representation aligned with practical priors, the construction of the density-equalizing mapping in Section~\ref{sec:de_map} requires a sufficiently smooth density. To this end, we regularize the piecewise constant function by applying a Gaussian filter. Specifically, we smooth the intermediate density $\hat{\rho}'$ to obtain a continuous density function $\rho$, which is then used to compute the density-equalizing mapping.

To find the corresponding density-equalizing mapping $f_\rho$, we apply the Density-Equalizing Solver (DES) based on the mapping solver in \cite{huang2025learning}. The network structure of DES is shown in Fig~\ref{fig:des}. The loss function is given by
\begin{equation}\label{loss_estimator}
    \mathcal{L}_\text{DES} = \lambda_\rho \mathcal{L}_\rho + \lambda_\text{BC} \mathcal{L}_\text{BC},
\end{equation}
where 
\begin{equation}
\begin{aligned}
    \mathcal{L}_\rho = \text{Var}(\rho)\\
    \mathcal{L}_\text{BC} = \|\mu_{f_\rho}\|_2^2
\end{aligned}
\end{equation}
and $\lambda_\rho$, $\lambda_\text{BC}$ are the corresponding weightings. The settings $\lambda_\rho = 1$, $\lambda_\text{BC} = 0.8$ are used for training the DES for this work. With the losses, the mapping is quasi-conformal with a careful control of $\lambda_\text{BC}$ so that $\|\mu_f\|^2<1$, encouraging the bijectivity. 

\begin{figure}[b]
    \centering
    \includegraphics[width=0.48\textwidth]{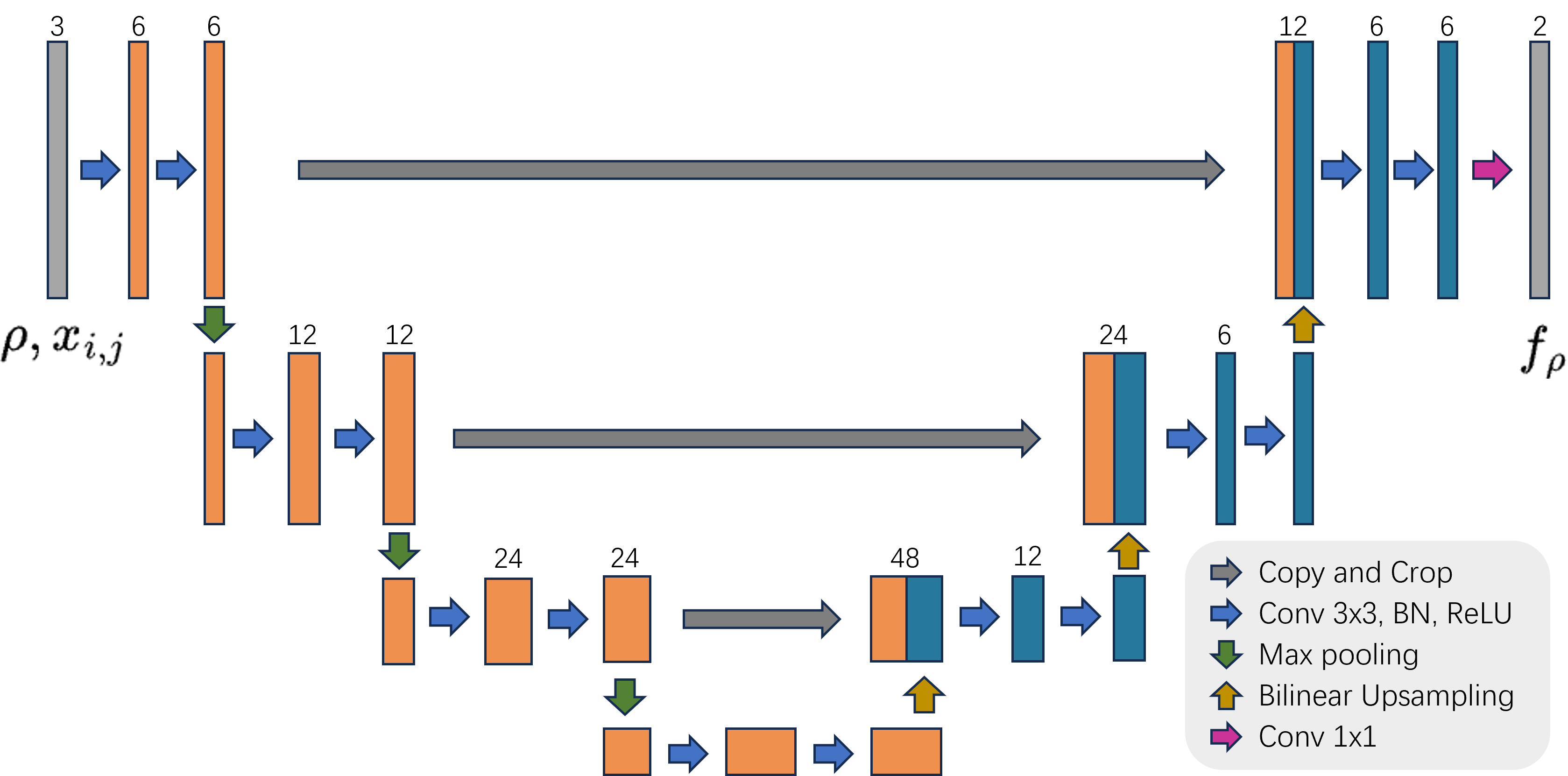}
    \caption{The network structure of Density-Equalizing Solver (DES).}
    \label{fig:des}
\end{figure}

\subsection{The Overall Model}

Then, we describe the overall model of the proposed framework, as summarized in Fig~\ref{fig:decnn}. We denote the collection of input signals for a given task by $\{h_i\}_{i=1}^Q$.

As introduced in Section~\ref{sec:de_map_sol}, a single shared density function $\rho$ is constructed from a set of trainable scalars $\{s_k\}$ defined piecewise over the partition $\{R_k\}$. A pretrained Density-Equalizing Solver ($\mathcal{N}_\text{DES}$) is then employed to the Gaussian smoothed function $\rho$ to compute the corresponding mapping:
\begin{equation}
\label{eq:gen_de_map}
    f_\rho = \mathcal{N}_\text{DES}(\rho).
\end{equation}

In parallel, each surface signal $h_i$ is parameterized via a mapping $\phi_i$, yielding the pullback function $\tilde{h}_i = h_i \circ \phi_i$ defined on $\Omega$. Each parameterization $\phi_i$ is enforced to align the regions on the manifold with the partition $\{R_k\}_{k=0}^K$ in the parameterized domain. We then compute the transformed function
\begin{equation}
    \hat{h}_i = \tilde{h}_i \circ f_\rho^{-1},
\end{equation}
which can be interpreted as a 2D warped representation of the original signal. As discussed in Remark~\ref{rmk:dec=2d}, performing a standard 2D convolution on the deformed image $\hat{h}_i$ is equivalent to applying a $\rho$-induced DEC on the parameterized signal $\tilde{h}_i$, whose receptive field is transformed by the density-equalizing mapping $f_\rho$, which is again equivalent to perform a deformable convolution over the original manifold image $h$.

Finally, for a given task, we denote the downstream network by $\mathcal{N}_\text{task}$, which produces the output
\begin{equation}
    \label{eq:task_net}
    y_i = \mathcal{N}_\text{task}(\hat{h}_i; \boldsymbol{\theta}),
\end{equation}
where $\boldsymbol{\theta}$ denotes the trainable parameters. $\mathcal{N}_\text{task}$ can be any 2D CNN with a task-dependent structure since $\hat{h}$ is planar. 
Similar to~\cite{zhang2025quasi}, a sequence of DEC can be shown to be equivalent to a single density-equalizing mapping applied at the input stage. The details are omitted here and deferred to Appendix~\ref{app:multi_dec}. 

To summarize, Fig~\ref{fig:decnn} presents the complete framework of DECNN, which is composed of three main stages. First, a density function $\rho$ is constructed from the estimated density parameters, and the associated DE mapping $f_\rho$ is computed using a pretrained solver $\mathcal{N}_\text{DES}$, as described in Section~\ref{sec:de_map_sol} and Eq~\eqref{eq:gen_de_map}. Next, the function $h$, originally defined on the surface, is conformally mapped onto a planar domain to obtain $\tilde{h}$, and subsequently deformed into a wrapped planar image $\hat{h}$. Finally, a task-dependent CNN is applied to the wrapped image $\hat{h}$ according to Eq~\eqref{eq:task_net}, producing the final output. The figure also highlights the equivalence between deformation and feature extraction throughout the processing pipeline, as established in Section \ref{sec:opt_disc}.

\subsection{The Training Process}

To train the model, we define a task-specific loss function $\mathcal{L}_\text{task}$, such as cross-entropy loss for classification or mean squared error for regression or segmentation. Since the density-equalizing solver is pretrained to produce stable mappings, we do not impose additional regularization directly on $f_\rho$. Instead, we regularize the density function and define the overall objective as
\begin{equation}
\begin{aligned}
\label{eq:loss}
    \mathcal{L}(\rho, \tilde{h}_{1:Q}; \boldsymbol{\theta})
    =\;& \lambda_\text{task} \mathcal{L}_\text{task}(\hat{h}_{1:Q}; \boldsymbol{\theta}) \\
    &+ \lambda_\text{peak} \mathcal{L}_\text{peak}(\rho)
    + \lambda_\text{scale} \mathcal{L}_\text{scale}(M),
\end{aligned}
\end{equation}
where $\lambda_\text{task}$, $\lambda_\text{peak}$, and $\lambda_\text{scale}$ are weighting coefficients that balance the contributions of each term. The regularization terms are defined as
\begin{equation}
\label{eq:p_loss}
\begin{aligned}
    \mathcal{L}_\text{peak}(\rho) &= \frac{1-\|\rho\|_\infty}{M}, \\
    \mathcal{L}_\text{scale}(M) &= M,
\end{aligned}
\end{equation}
where $\mathcal{L}_\text{peak}$ encourages the density to concentrate, while $\mathcal{L}_\text{scale}$ penalizes large scaling factors to promote smoother density-equalizing mappings.

The model is trained via gradient-based optimization, jointly learning the downstream network parameters $\boldsymbol{\theta}$ and the density parameters $\{s_k\}$ that define $\rho$. As a result, the framework simultaneously performs the target task while learning the density function as a shared saliency map that reflects the importance of different regions.

\section{Experiments}
\label{sec:experiments}
This section details on our experiments. We demonstrate the performance of the proposed method in a MNIST classification task and craniofacial analysis problem.

\subsection{Experimental Setting}
%
The training is conducted on a computing server equipped with four Intel Xeon Gold 6252 24-core CPUs and two NVIDIA A40 Tensor Core GPUs.
%
The loss function is given by Eq~\eqref{eq:loss},
%
%
where we set $\lambda_\text{task} = 1$, $\lambda_\text{peak} = 10^{-4}$, and $\lambda_\text{scale} = 10^{-6}$.
The training parameters are set as follows: learning rate is $1.0 \times 10^{-5}$, batch size is $5000$ if not specified.
%
The baseline model, CNN, is introduced to compare classification accuracy, where there is no geometric adjustment.

\subsection{Task-aware Sampling in Image Classification}

\begin{figure}[t]
    \centering
    \includegraphics[width=0.45\textwidth]{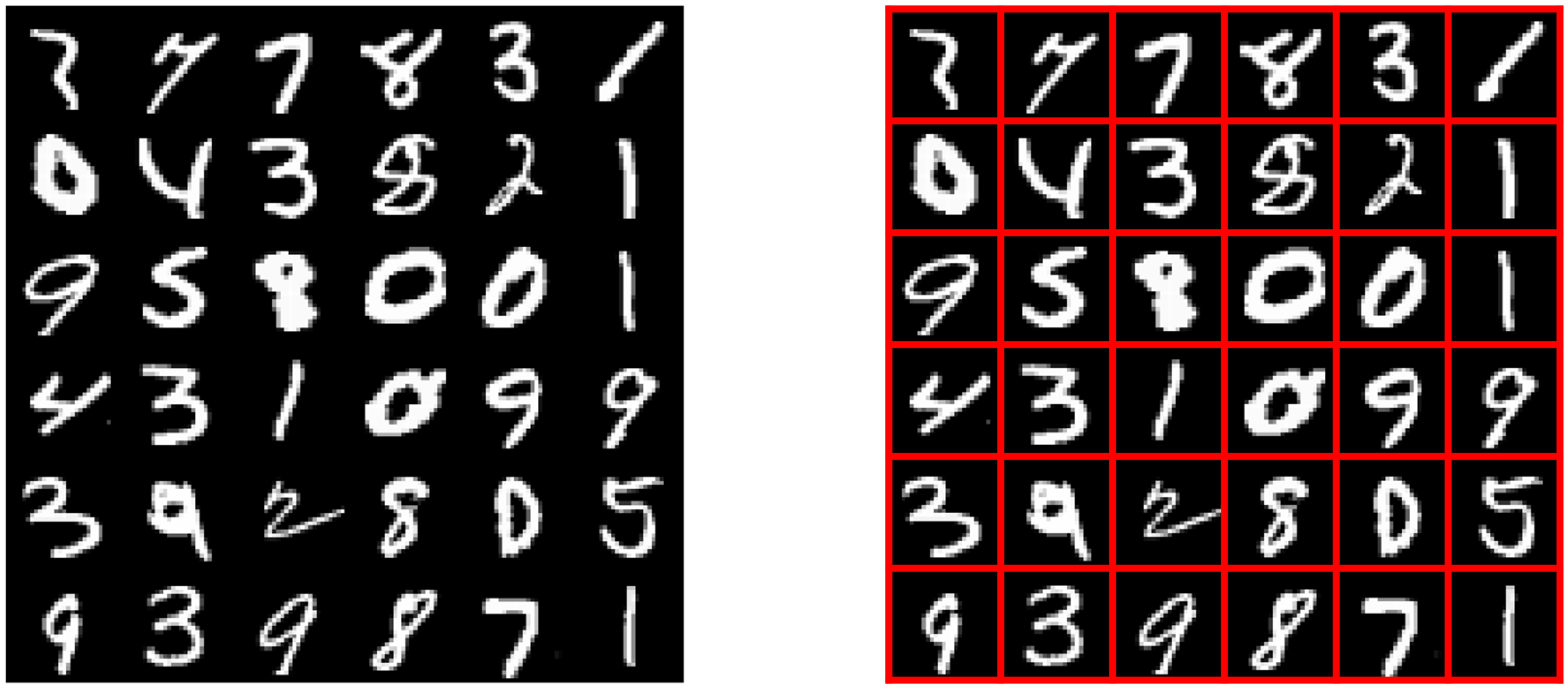}
    \caption{MNIST digits in a $6 \times 6$ grid with a region prior $\mathcal{R}$.}
    \label{fig:mnist_init}
\end{figure}
\begin{figure}[b]
    \centering
    \includegraphics[width=0.48\textwidth]{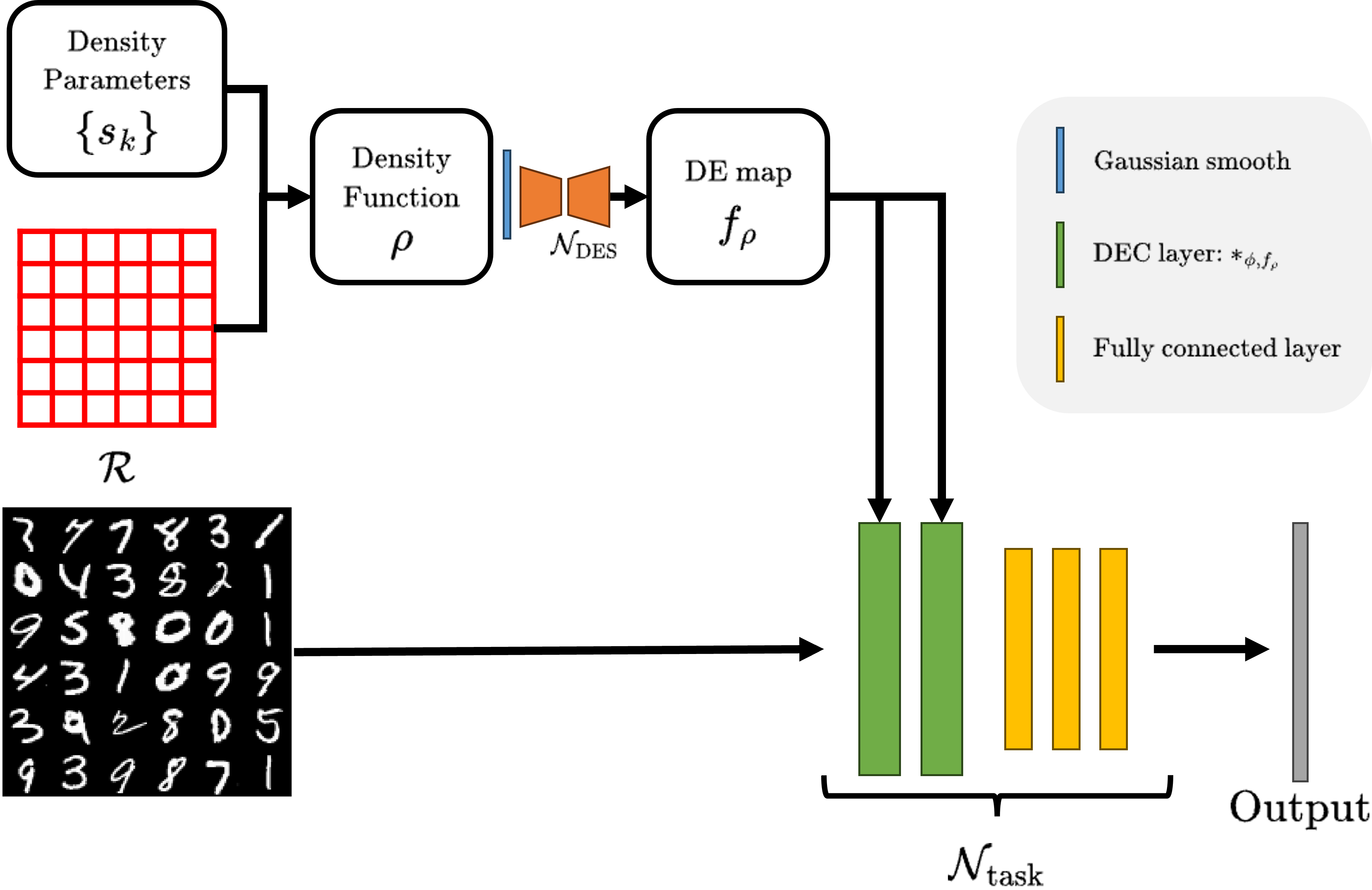}
    \caption{DECNN network architecture applied in the MNIST classification task.}
    \label{fig:mnist_net}
\end{figure}

In this section, we evaluate the performance of our proposed methods on image classification tasks involving MNIST digits. The MNIST dataset consists of 60,000 handwritten digits spanning ten classes (0-9), with image size $28 \times 28$. To evaluate the effectiveness of DECNN, we embed each digit image into the partitioned grid illustrated in Fig~\ref{fig:mnist_init}, where each block corresponds to a partitioned region defined in Eq~\eqref{eq:partition}. The resulting images are then used as inputs to the model.

The network architecture is illustrated in Fig~\ref{fig:mnist_net}. Initially, a smoothed piecewise density function $\rho$ is produced using the region prior $\mathcal{R}$, the density parameters $\{s_k\}$, and the activation function in Eq~\eqref{eq:p_loss} as described in Section~\ref{sec:de_map_sol}. Subsequently, the input source image, sized $N_0 \times N_0$, is deformed with a density-equalizing mapping $f_\rho$, generated by the pretrained Density-Equalizing Solver $\mathcal{N}_{DES}$ using the density function $\rho$. This deformed image is then subsampled into a size of $\lceil N_0/d \rceil \times \lceil N_0/d \rceil$ to serve as the input feature map for the classifier $\mathcal{N}_\text{task}$, where $d \in [1,2,3,4,5,6]$ is the sampling ratio. The classifier $\mathcal{N}_\text{task}$ consists of two convolutional layers and three fully connected layers. During the training process, we first train the density parameters $\{s_k\}$ and the classifier $\mathcal{N}_\text{task}$ together for 1000 epochs, followed by an additional 1000 epochs during which only the classifier $\mathcal{N}_\text{task}$ is trained. Eq~\eqref{eq:loss} is used as the loss function for this classification task, where the task loss $\mathcal{L}_\text{task}$ is defined as the cross entropy loss for classification.


Since the target label corresponds to only one digit in one grid block, the model is expected to identify the corresponding region as the key region and thus sample more feature points there. For example, if the label is designated as the top-left digit, then the key region should be the top-left area of the grid. In this experiment, we select the center square as the target region to minimize errors introduced by boundary conditions in density-equalizing mappings. 

We conduct the classification experiments with grid sizes ranging from $3 \times 3$ to $6 \times 6$ and sampling ratios $d$ ranging from 1 to 6. To evaluate the performance, the density ratio indicating the relative density value $r_{s'} = s'_\text{TR}/ \max \{s'_\text{NTR}\}$ is used, where $s'_\text{TR}$ is the predicted density for the target region, and $s'_\text{NTR}$ can be the predicted densities for non-target regions excluding the background.

\begin{table*}[t]
    \centering
    \begin{tabular}{C|c|CcC|BB|r}
    \toprule
    \multicolumn{2}{c|}{Setup}& \multicolumn{3}{c|}{Density values} & \multicolumn{2}{c|}{Accuracy (\%)} & \multicolumn{1}{c}{\% of}  \\
    Grid & $d$ & $s'_\text{TR}$ & Range$(\{s'_\text{NTR}\})$ & $r_{s'}$ & CNN & DECNN & Param. \\
    \hline
    \hline
    \multirow{6}{*}{$3 \times 3$} & 1 & 2.8731 & [1.0101, 1.1607] & 2.4752 & \textbf{98.92} & 98.68 & 100.00 \\
    & 2 & 3.6932 & [1.0018, 1.2939] & 2.8542 & 98.50 & \textbf{98.78} & 19.29 \\
    & 3 & 4.2281 & [1.0075, 1.5936] & 2.6532 & 96.12 & \textbf{98.57} & 8.70 \\
    & 4 & 4.7776 & [1.0000, 2.0305] & 2.3529 & 88.52 & \textbf{97.32} & 4.35 \\
    & 5 & 5.4743 & [1.0044, 1.9581] & 2.7957 & 77.62 & \textbf{95.37} & 2.99 \\
    & 6 & 4.6347 & [1.0000, 1.2679] & 3.6555 & 66.90 & \textbf{91.92} & 2.17 \\
    \hline
    \multirow{6}{*}{$4 \times 4$} & 1 & 4.7192 & [1.0000, 1.8355] & 2.5711 & \textbf{98.86} & 98.69 & 100.00 \\
    & 2 & 4.4858 & [1.0000, 1.1350] & 3.9522 & 98.48 & \textbf{98.76} & 22.11 \\
    & 3 & 4.5527 & [1.0000, 1.1667] & 3.9022 & 95.93 & \textbf{98.65} & 8.20 \\
    & 4 & 5.1839 & [1.0000, 1.3590] & 3.8144 & 88.85 & \textbf{97.26} & 4.69 \\
    & 5 & 5.0229 & [1.0000, 1.5491] & 3.2424 & 77.64 & \textbf{96.08} & 2.34 \\
    & 6 & 6.7449 & [1.0000, 1.4958] & 4.5092 & 65.34 & \textbf{93.50} & 1.61 \\
    \hline
    \multirow{6}{*}{$5 \times 5$} & 1 & 2.4509 & [1.0002, 1.1279] & 2.1729 & \textbf{98.90} & 98.63 & 100.00 \\
    & 2 & 2.6048 & [1.0012, 1.0600] & 2.4572 & 98.40 & \textbf{98.67} & 21.17 \\
    & 3 & 3.6003 & [1.0017, 1.1832] & 3.0428 & 95.69 & \textbf{98.25} & 8.03 \\
    & 4 & 3.4711 & [1.0066, 1.2014] & 2.8892 & 88.45 & \textbf{96.58} & 3.92 \\
    & 5 & 4.0144 & [1.0025, 1.2515] & 3.2076 & 76.08 & \textbf{95.20} & 2.92 \\
    & 6 & 4.3193 & [1.0027, 1.4486] & 2.9817 & 65.33 & \textbf{92.87} & 2.10 \\
    \hline
    \multirow{6}{*}{$6 \times 6$} & 1 & 3.0909 & [1.0000, 1.4946] & 2.0681 & \textbf{98.99} & 98.63 & 100.00 \\
    & 2 & 2.9882 & [1.0000, 1.3470] & 2.2184 & 98.53 & \textbf{98.65} & 22.90 \\
    & 3 & 3.0000 & [1.0000, 1.1742] & 2.5549 & 95.47 & \textbf{97.93} & 9.40 \\
    & 4 & 3.5289 & [1.0003, 1.1007] & 3.2061 & 87.76 & \textbf{97.33} & 4.42 \\
    & 5 & 3.4462 & [1.0000, 1.0917] & 3.1567 & 78.02 & \textbf{95.37} & 2.68 \\
    & 6 & 3.6990 & [1.0001, 1.0797] & 3.4259 & 64.00 & \textbf{90.20} & 1.99 \\
    \hline
    \bottomrule
    \end{tabular}
    \caption{Quantitative results of MNIST image classification: predicted density $s_k'$ in target regions (TR) and non-target regions (NTR), the density ratio $r_{s'}$, accuracies of CNN and DECNN, and \% of parameters.}
    \label{tb:mnist_s_krd}
\end{table*}

\begin{figure*}[t]
    \centering
    \includegraphics[width=.95\textwidth]{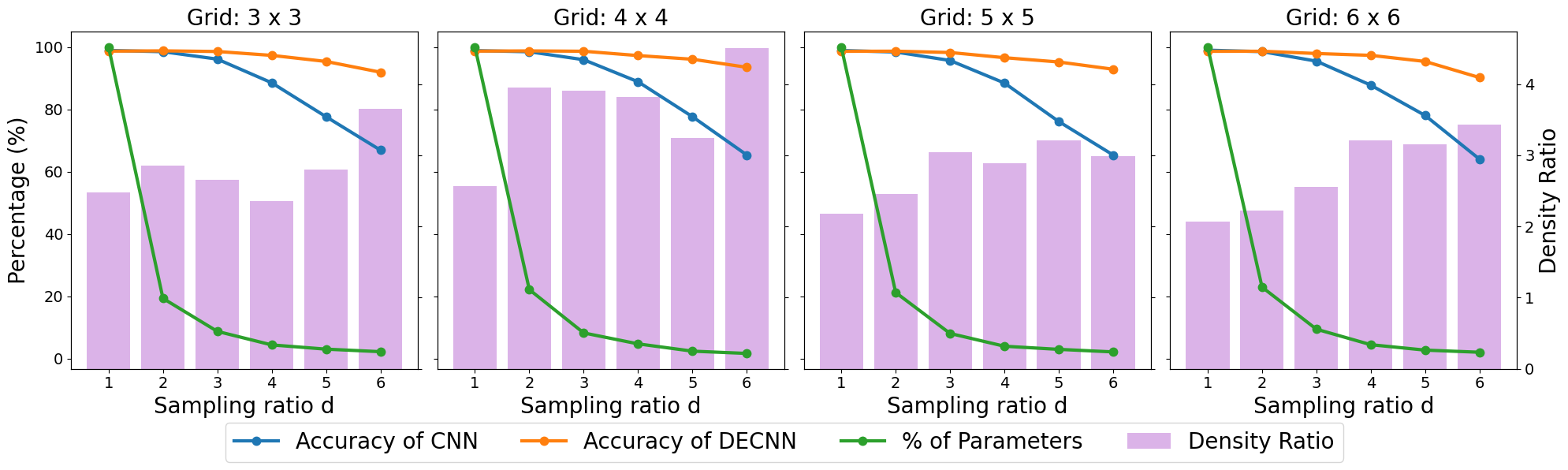}
    \caption{Graphical result of MNIST image classification: density ratio $r_{s'}$, accuracies of CNN and DECNN, and \% of parameters.}
    \label{fig:mnist_result_chart_c2}
\end{figure*}

Notably, as shown in Table~\ref{tb:mnist_s_krd}, large predicted density ratios ($r_{s'} > 2$) are achieved, indicating a clear distinction of the target regions to the rest, across all setups. By Theorem \ref{thm:DE_sampling}, A high density ratio $r_{s'}$ allows a denser sampling in the target region, capturing more informative feature points to facilitate successive learning. Besides, it can be observed that the produced density values in most target regions exceed 2.5, whereas those in the remaining areas are generally below 1.5. This clear separation suggests that the target regions are consistently enlarged across different setups, indicating that the deformation selectively emphasizes areas of interest.

To illustrate the advantages conferred by the optimized transformation based on the inferred density function, we compare the accuracies achieved by CNN and DECNN. Since both models take the same resolution of input images, they contain roughly the same number of learnable parameters at each sampling ratio. Table~\ref{tb:mnist_s_krd} reports the classification accuracies for each experimental configuration, together with the number of learnable parameters expressed as a percentage of the largest model (those with $d=1$) for each partitioning grid size. As the sampling ratio $d$ increases, the number of learnable parameters decreases sharply across all grid settings. Despite this substantial reduction, DECNN maintains an accuracy above 90\% in every setup, even at the highest sampling ratio ($d=6$), where it uses only about 2\% of the parameters required by the corresponding largest models ($d=1$).

By contrast, the standard CNN exhibits a marked decline in performance as the sampling ratio increases, with accuracy dropping below 65\% due to the absence of an optimized sampling mechanism. To retain an accuracy above 90\%, the CNN requires a sampling ratio no greater than 3. In comparison, the proposed DECNN continues to achieve strong performance even when the sampling ratio reaches $d=6$. These results indicate that DECNN's adaptive convolution effectively extracts the most important information from the high-resolution source data, surpassing the baseline CNN model. 

For clearer presentation of the experimental results, these values are illustrated in Fig~\ref{fig:mnist_result_chart_c2}. Four separate charts are provided, each corresponding to a different grid configuration. As the sampling ratio ($x$-axis) increases, the number of parameters (green curves) significantly decreases across all graphs. 
At the same time, the accuracy of CNNs (blue curves) declines sharply, whereas DECNN (orange curves) shows a minor decrease across all configurations. Meanwhile, the resulted density ratio (purple bars) exhibits a slight upward trend as the sampling ratio increases. This behavior reflects the growing need for optimized feature extraction when fewer feature points are available.
Since the boundary conditions of density-equalizing mappings impose weaker deformation constraints on interior regions in larger grids, this effect becomes more noticeable as the grid size increases.

A visualization of the results for the setup with a grid size of $6 \times 6$ and a sampling ratio of $d=6$ is shown in Fig~\ref{fig:mnist_result}. The top-left image displays the original source image under the identity mapping. Without deformation, as illustrated in the bottom-left image, the feature image is uniformly subsampled across the entire domain, causing important structures to be severely degraded and making the digit difficult to distinguish. In contrast, the right column shows the effect of the learned density-equalizing mapping. The source image is adaptively deformed to produce the warped image shown in the top-right panel, where the informative region is automatically emphasized despite the absence of any region-level supervision during training. After the same subsampling process is applied, the resulting image in the bottom-right panel still preserves the recognizable structure of the digit ``8''. This demonstrates the geometric mechanism of the proposed architecture, where informative regions are enlarged through the learned DE transformation, enabling more effective feature extraction under limited sampling resolution.

\begin{figure}[t]
    \centering
    \includegraphics[width=0.45\textwidth]{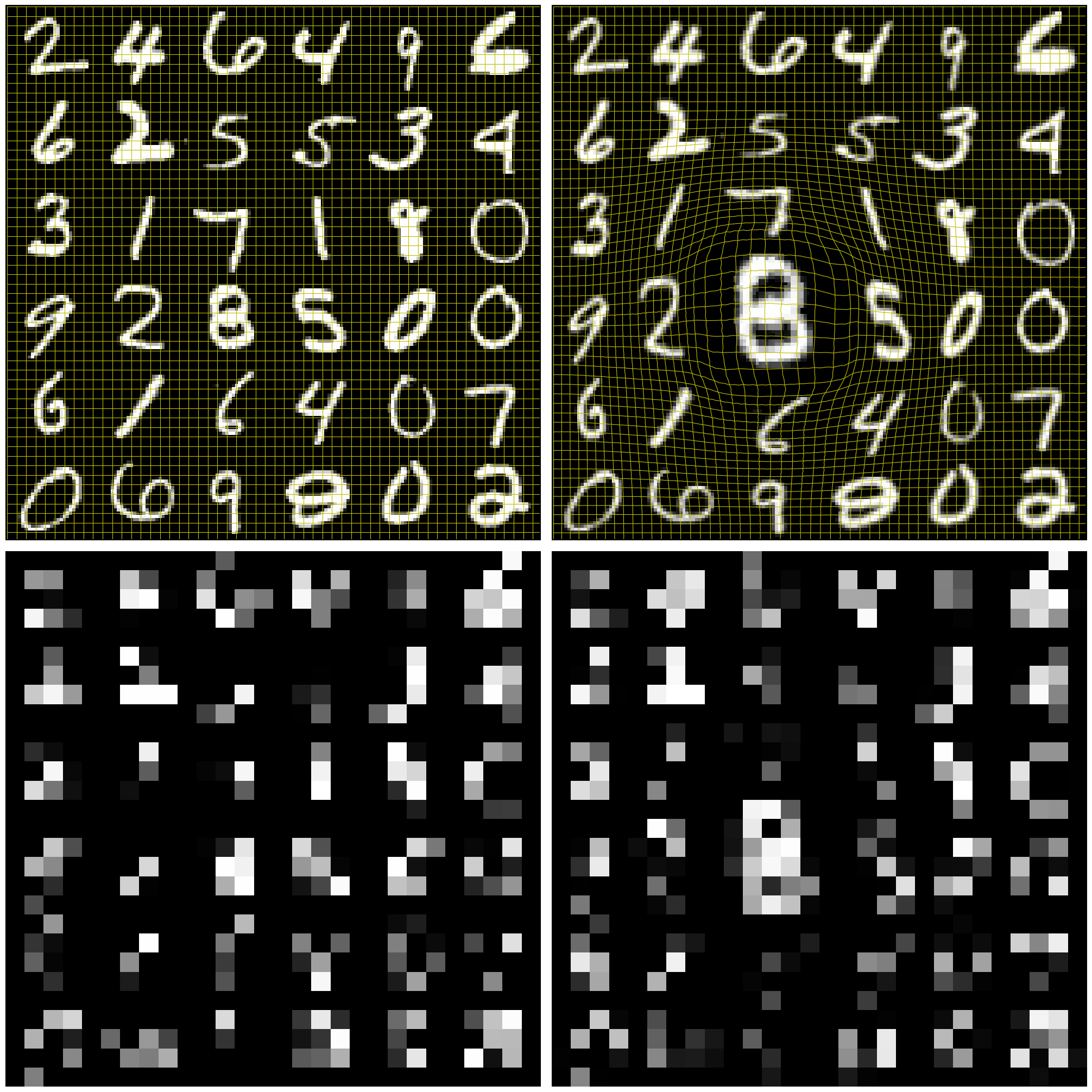}
    \caption{Visualization of result for a $6 \times 6$ grid with $d=6$.}
    \label{fig:mnist_result}
\end{figure}



\subsection{Craniofacial Analysis}
\label{sec:facial}

\begin{figure}[b]
    \centering
    \includegraphics[width=0.45\textwidth]{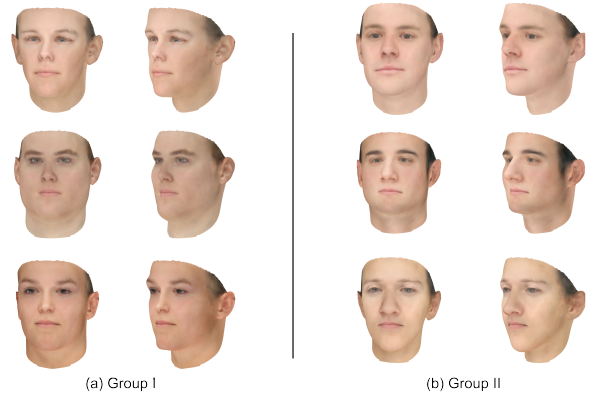}
    \caption{Craniofacial analysis examples for faces with different nasal structures.}
    \label{fig:faceosa}
\end{figure}

\begin{figure}[t]
    \centering
    \includegraphics[width=0.48\textwidth]{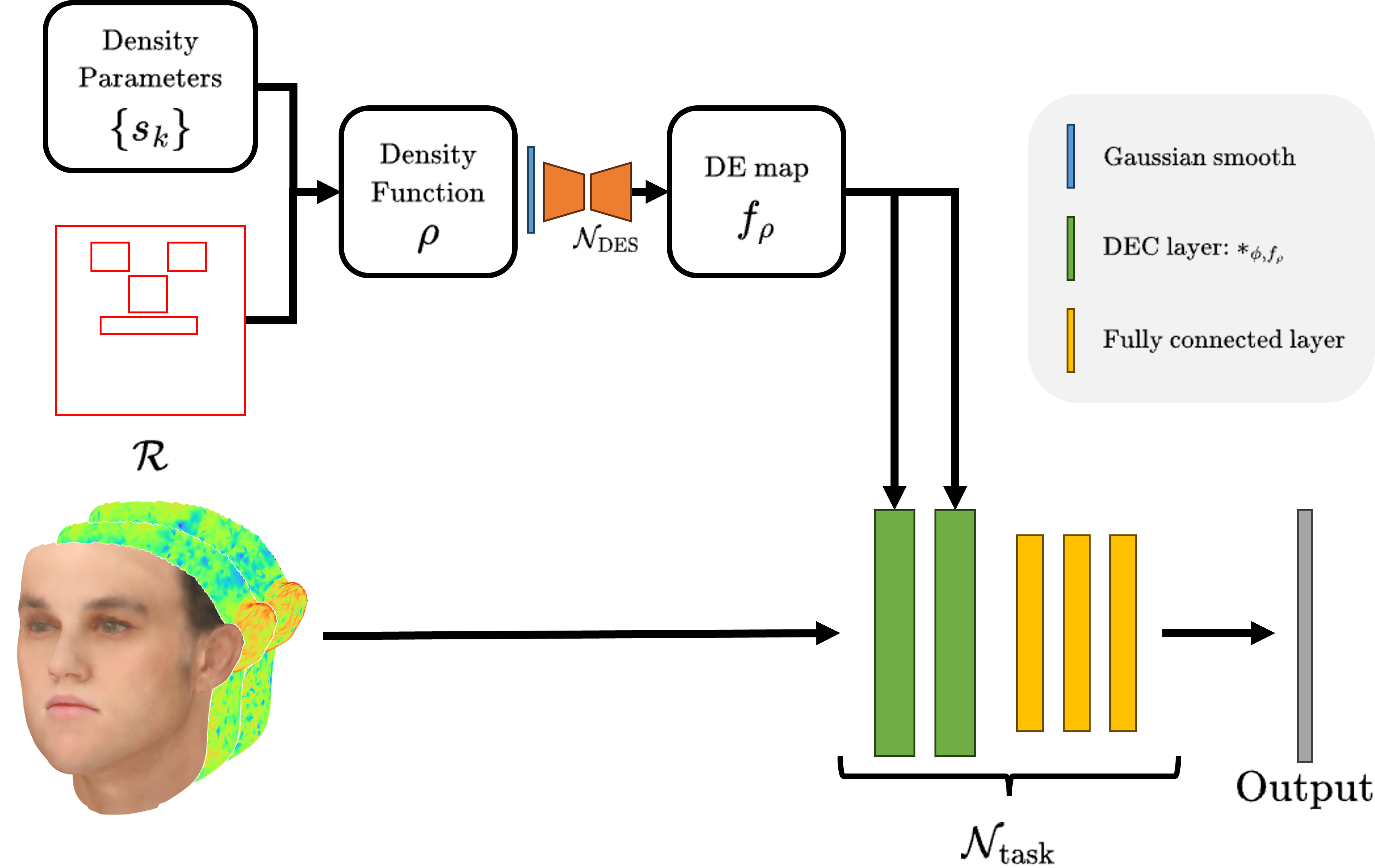}
    \caption{Craniofacial Analysis Network Architecture.}
    \label{fig:osa_net}
\end{figure}

Craniofacial analysis is a multi-disciplinary study of the structural, architectural, and anatomical features of the human head and face, with important applications in orthodontics, orthognathic surgery, syndrome diagnosis, forensic science, and medical image analysis. For instance, nasal structures can be investigated to evaluate facial asymmetries, plan surgical interventions, and provide essential clinical indicators for various orthodontic treatments \cite{nehra2009nasal, burstone1958integumental} and disease diagnoses \cite{matthews2022static}. To better analyze these data using machine learning models, extracting more feature points from informative regions is essential, especially when those regions are relatively small compared to the whole domain. Besides, some diseases may produce distinct structural patterns that are not yet recognized by medical professionals. In such cases, our proposed model could identify key regions critical for doctors to observe, thereby aiding in the determination of whether a patient possesses the disease. This capability is particularly valuable in clinical practice.

Three datasets of 2D human facial surfaces are used, comprising two groups with 500 subjects each. Each surface consists of 1024 vertices and exhibits distinct structural patterns in fixed sensory organs, which may include variations in size, symmetry, and other characteristics. The corresponding sensory organs for the three datasets are the pair of eyes, the nose, and the mouth, respectively. In each dataset, 70\% of the samples are used for training, while the remaining 30\% are reserved for testing. Representative examples from the two groups, highlighting different nasal structural patterns, are presented in Fig~\ref{fig:faceosa}. These three regions possess distinctive shapes: while the nose can be enclosed by a rectangle, the mouth is bounded by one with a higher width-to-height ratio, and the eyes are represented by two disconnected squares.


For surface images, it is necessary to first parameterize each onto a 2D domain. As mentioned in Section~\ref{sec:opt_disc}, we flatten the facial mesh using conformal parameterization \cite{meng2016conformal}. In this experiment, we map the meshes into images of size $256 \times 256$, with subsequent subsampling reducing the resolution to $32 \times 32$. To retain 3D geometric information, Gaussian curvature and mean curvature are computed \cite{Dastan2025} and mapped, along with the texture information, to the 2D image, which serves as the high-resolution input for the network. After the preprocessing stage, the setup remains consistent with the previous MNIST experiment to ensure a straightforward workflow. The classifier consists of two convolutional layers and three fully connected layers, as depicted in Fig~\ref{fig:osa_net}, and the region prior frames three regions: the eyes, nose, and mouth.

A batch size of 50 is used for training. During the training process, we train the density parameters $\{s_k\}$ and the classifier $\mathcal{N}_\text{task}$ together for 2000 epochs. Eq~\eqref{eq:loss} is used as the loss function for this classification task, where the task loss $\mathcal{L}_\text{task}$ is defined as the binary cross-entropy loss for classification.

\begin{figure}[t]
    \centering
    \includegraphics[width=0.45\textwidth]{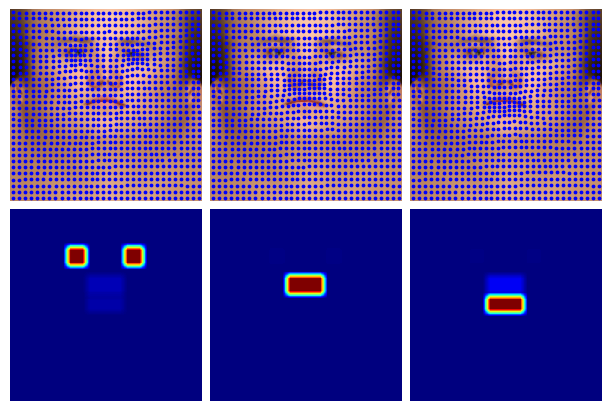}
    \caption{Craniofacial analysis visualization. The left, middle, and right columns show the results when the eyes, nose, and mouth, respectively, are the target regions. The first row illustrates the sampling locations on the input image, and the second row is the learned density heatmap.}
    \label{fig:osa_visual}
\end{figure}
\begin{table}[t]
\centering
\begin{tabular}{c|ccc|c}
\toprule
Target Region  & $s'_\text{eye}$ & $s'_\text{nose}$ & $s'_\text{mouth}$ & $r_{s'}$  \\\hline
Eye    & \bf 2.0133     & 1.0490     & 1.0360 & 1.9345     \\ 
Nose   & 1.0045 & \bf 1.8841 & 1.0000 & 1.8755 \\ 
Mouth  & 1.0046 & 1.1074     & \bf 2.0507 & 1.8518     \\\hline
\bottomrule
\end{tabular}
\caption{Quantitative results of craniofacial analysis: predicted density value in target regions (TR) and non-target regions (NTR), and the density ratio $r_{s'}$.}
\label{tb:osa_krd}
\end{table}

\begin{table*}[t]
\centering
\begin{tabular}{c|ccccc}
\toprule
Method  & Accuracy  & Sensitivity & Precision    & N.P.V.   & Specificity   \\\hline
GCNN    & 89.33     & 91.33     & 87.82     & 90.97         & 87.33\\ 
PTCNet  & 91.00     & 92.00     & 90.20     & 91.84         & 90.00\\ 
DECNN & \textbf{95.67}     & \textbf{96.00}     & \textbf{95.36}     & \textbf{95.97}         & \textbf{95.33} \\
\hline
\bottomrule
\end{tabular}
\caption{Quantitative comparison of craniofacial analysis results.}
\label{tb:osa_acc}
\end{table*}

For each predefined region, the target region corresponds to the area containing the relevant sensory organ. Since different classes in the dataset exhibit distinct structural characteristics, the learned density function is expected to emphasize regions that contain discriminative anatomical information. Table~\ref{tb:osa_krd} shows that DECNN successfully identifies these informative regions, achieving a pronounced density ratio $r_{s'}$ that clearly differentiates the target region from the remaining areas. Furthermore, as illustrated in Fig~\ref{fig:osa_visual}, the learned DE mapping enlarges informative regions, enabling more feature samples to be allocated to these areas during convolution. The density heatmaps shown at the bottom further confirm that the learned density values in the target regions are significantly higher than those in less informative regions.

Table~\ref{tb:osa_acc} presents the classification performance of DECNN in comparison with GCNN and PTCNN. Despite operating on substantially sparser sampled features, the proposed framework still achieves superior classification accuracy over the competing methods. These results demonstrate that the learned density-equalizing convolution effectively guides the network toward extracting the most discriminative information from the data, leading to more efficient and task-adaptive feature learning.

\section{Conclusion}

In this work, we introduce Density-Equalizing Convolution (DEC) for simply connected open surfaces, which leverages a learnable density function to encode spatial importance across the domain. Through theoretical analysis, we show that the associated density-equalizing mapping redistributes sampling points according to this density, such that regions with higher density values receive more samples. This result establishes that the density function serves not only as an intuitive indicator of importance but also as a mathematically grounded measure that directly governs the allocation of sampling resources.

Building upon DEC, we develop the Density-Equalizing Convolutional Neural Network (DECNN), which integrates this adaptive sampling mechanism into deep learning. By operating on a reparameterized domain induced by the density-equalizing mapping, DECNN enables convolution to focus on informative regions while reducing redundancy in less relevant areas, leading to more efficient use of model capacity.

We validate the proposed framework through experiments on both image classification and craniofacial surface analysis tasks. The results demonstrate that DECNN is able to learn meaningful importance maps, effectively concentrate sampling on task-relevant regions, and achieve competitive or superior performance with fewer parameters.

In future work, we plan to extend the proposed framework beyond simply connected surfaces to more general manifolds with complex topology, where defining global parameterizations becomes more challenging. From a modeling perspective, integrating DEC with modern neural architectures, such as attention mechanisms and transformer-based models, could further enhance its adaptability and representation power. In addition, we aim to explore broader applications of the proposed method across diverse domains, including medical image analysis, 3D shape understanding, and general computer vision tasks, where informative features are often sparsely distributed.

\section*{Declarations}
\textbf{Funding} This work was supported by HKRGC GRF (Project ID: 14306721), and Hong Kong Centre for Cerebro- Cardiovascular Health Engineering (COCHE).
\textbf{Conflict of Interest} The authors have no competing interests to declare that are relevant to the content of this article.

\bibliography{references}

\newpage
\appendix
\input{Appendix}

\end{document}

%% file: Appendix.tex
\section{Multiple Density-Equalizing Layers}
\label{app:multi_dec}
A sequence of density-equalizing convolution (DEC) layers can be shown to be equivalent to a sequence of 2D convolutional layers with a single density-equalizing transformation applied at the input stage. To formalize this idea, we first consider a network composed of multiple DEC layers.

Let $\rho^{(l)}$ denote the density function at layer $l$, which induces the density-equalizing mapping $f_{\rho}^{(l)}$. Given an input manifold signal $h_{\mathrm{in}}$, we iteratively apply DEC operations with kernels $k^{(l)}$ to compute feature maps $h^{(l)}$. After each convolution, a nonlinear activation $\sigma$ is applied. The forward propagation can be written as
\begin{equation}
\begin{aligned}
    h^{(1)} &= h_{\mathrm{in}},\\
    h^{(2)} &= \sigma\!\left(h^{(1)} \ast_{\phi, f_{\rho}^{(1)}} k^{(1)}\right),\\
    h^{(3)} &= \sigma\!\left(h^{(2)} \ast_{\phi, f_{\rho}^{(2)}} k^{(2)}\right),\\
    &\qquad \vdots\\
    h_{\mathrm{out}} &= \sigma\!\left(h^{(L)} \ast_{\phi, f_{\rho}^{(L)}} k^{(L)}\right),
\end{aligned}
\label{eq:decnn}
\end{equation}
where $\ast_{\phi, f_{\rho}^{(l)}}$ denotes the DEC operator defined in Definition~\ref{def:de_conv}, and $\sigma$ is the activation function.

In practice, an effective density function $\rho$ that represents saliency should depend solely on the original dataset and the learning objective, rather than on intermediate network outputs. Therefore, it is natural to impose the same DE mapping $f_\rho$ across all layers, 
i.e., $f_{\rho}^{(l)} \equiv f_\rho$ for all $l$. Under this assumption, and following the formulation in Eq~\eqref{eq:de_conv}, Eq~\eqref{eq:decnn} can be rewritten as
\begin{equation}
\begin{array}{rll}
    h^{(1)} &= h_{\mathrm{in}} &= \hat{h}_{\mathrm{in}} \circ f_\rho \circ \phi^{-1}, \\
    h^{(2)} &= \sigma\!\left(h^{(1)} \ast_{\phi, f_\rho} k^{(1)}\right) &= \sigma\!\left(\hat{h}^{(1)} \ast k^{(1)}\right) \circ f_\rho \circ \phi^{-1}, \\
    h^{(3)} &= \sigma\!\left(h^{(2)} \ast_{\phi, f_\rho} k^{(2)}\right) &= \sigma\!\left(\hat{h}^{(2)} \ast k^{(2)}\right) \circ f_\rho \circ \phi^{-1}, \\
    &\qquad \vdots \\
    h_{\mathrm{out}} &= \sigma\!\left(h^{(L)} \ast_{\phi, f_\rho} k^{(L)}\right) &= \sigma\!\left(\hat{h}^{(L)} \ast k^{(L)}\right) \circ f_\rho \circ \phi^{-1},
\end{array}
\label{eq:decnn_flat}
\end{equation}
where $\hat{h}^{(l)} = h^{(l)} \circ \phi \circ f_\rho^{-1}$ denotes the warped feature representation in the parameter domain.

Rearranging Eq~\eqref{eq:decnn_flat} yields
\begin{equation}
\begin{aligned}
    \hat{h}^{(2)} &= \sigma\!\left(\hat{h}^{(1)} \ast k^{(1)}\right), \\
    \hat{h}^{(3)} &= \sigma\!\left(\hat{h}^{(2)} \ast k^{(2)}\right), \\
    &\qquad \vdots \\
    \hat{h}_{\mathrm{out}} &= \sigma\!\left(\hat{h}^{(L)} \ast k^{(L)}\right),
\end{aligned}
\label{eq:decnn_flat2}
\end{equation}
which is precisely a standard sequence of 2D convolutions applied to the warped feature maps.

While Remark~\ref{rmk:dec=2d} establishes the equivalence between a single DEC and a 2D convolution on a warped feature map, Eq~\eqref{eq:decnn_flat2} further shows that a sequence of DEC layers on the manifold can be fully represented by a standard 2D CNN operating on $\hat{h}$. Consequently, repeatedly mapping features between the manifold and the parameter domain is unnecessary.

This equivalence significantly reduces computational complexity by avoiding repeated parameterization and inverse mappings. More importantly, it provides a practical way to convert any standard 2D CNN into a DEC-based network. Specifically, given a manifold signal $h$ and a 2D CNN $\mathcal{N}$, preprocessing the input as $\hat{h} = h \circ \phi \circ f_\rho^{-1}$ allows $\mathcal{N}(\hat{h})$ to approximate the output of a DEC network applied to $h$. Conversely, implementing a DEC network can be achieved by warping the input once and applying a conventional 2D CNN, as illustrated in Fig~\ref{fig:decnn}. In this implementation, pooling operations on the manifold are naturally realized as standard 2D pooling in the parameter domain.